\documentclass[letterpaper]{article} 
\usepackage{aaai23} 
\usepackage{times}  
\usepackage{helvet} 
\usepackage{courier}  
\usepackage[hyphens]{url}  
\usepackage{graphicx} 
\usepackage{natbib}  
\usepackage{caption} 
\usepackage[ruled,vlined]{algorithm2e}
\usepackage{amsthm}
\usepackage{multirow}
\newtheorem{theorem}{Theorem}
\newtheorem{lemma}{Lemma}
\newtheorem{statement}{Statement}
\usepackage{graphicx}
\usepackage{caption}
\usepackage{subcaption}
\usepackage[toc,page]{appendix}
\usepackage{siunitx}
\usepackage{xcolor}

\usepackage{amsfonts} 
\title{Safe Interval Path Planning With Kinodynamic Constraints}

\author {
    Zain Alabedeen Ali,\textsuperscript{\rm 1}
    Konstatin Yakovlev \textsuperscript{\rm 2, 3, 1}
}
\affiliations {
    \textsuperscript{\rm 1} Moscow Institute of Physics and Technology, Moscow, Russia\\
    \textsuperscript{\rm 2} Federal Research Center for Computer Science and Control of Russian Academy of Sciences, Moscow, Russia\\
    \textsuperscript{\rm 3} AIRI, Moscow, Russia\\
    ali.za@phystech.edu, yakovlev@isa.ru
}

\begin{document}

\maketitle

\begin{abstract}
Safe Interval Path Planning (SIPP) is a powerful algorithm for solving single-agent pathfinding problem when the agent is confined to a graph and certain vertices/edges of this graph are blocked at certain time intervals due to dynamic obstacles that populate the environment. Original SIPP algorithm relies on the assumption that the agent is able to stop instantaneously. However, this assumption often does not hold in practice, e.g. a mobile robot moving with a cruising speed is not able to stop immediately but rather requires gradual deceleration to a full stop that takes time. In other words, the robot is subject to kinodynamic constraints. Unfortunately, as we show in this work, in such a case original SIPP is incomplete. To this end, we introduce a novel variant of SIPP that is provably complete and optimal for planning with acceleration/deceleration. In the experimental evaluation we show that the key property of the original SIPP still holds for the modified version -- it performs much less expansions compared to A* and, as a result, is notably faster. 
\end{abstract}

\section{Introduction}

\begin{figure}[t]
  \centering
\includegraphics[width=0.95\linewidth]{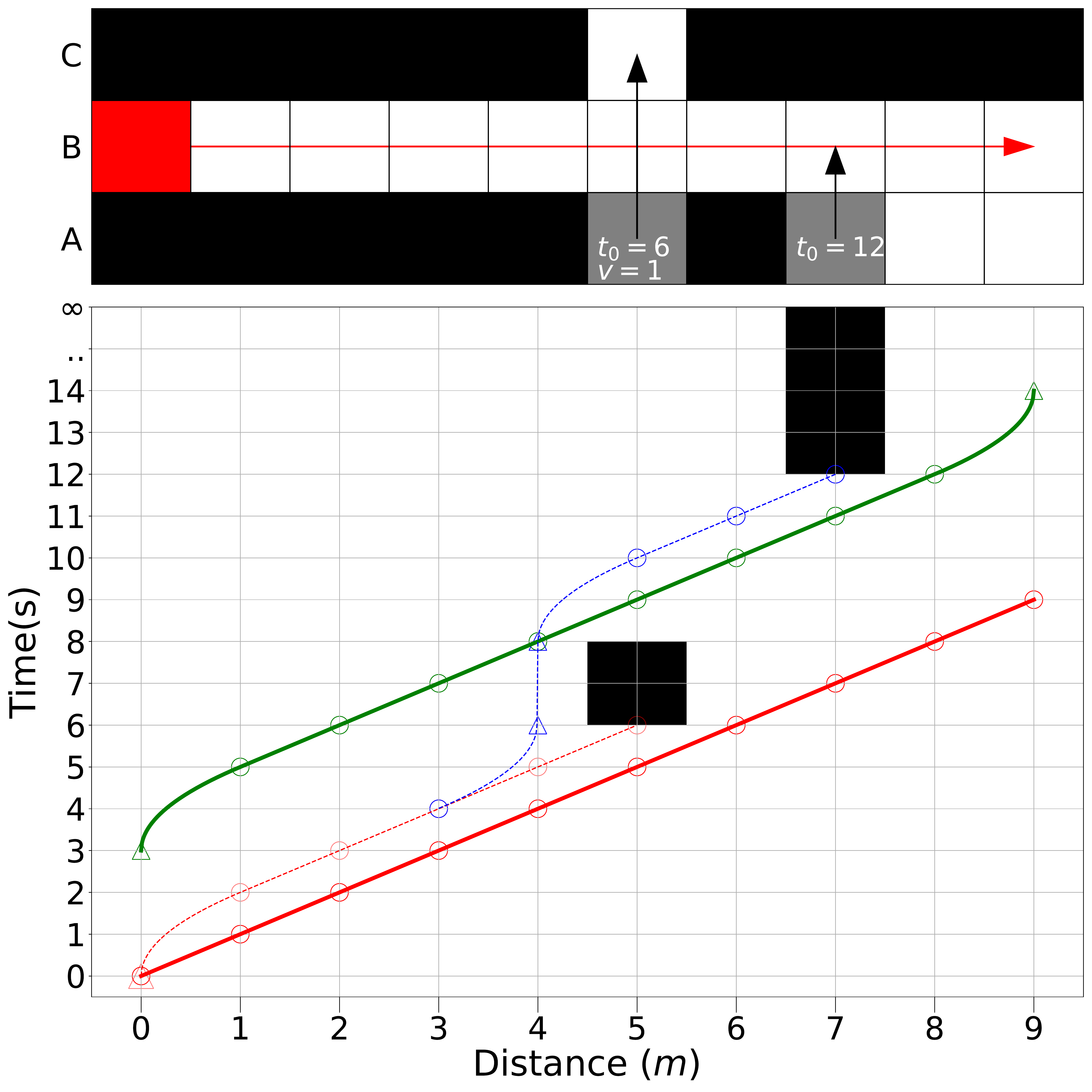}
\caption{
An example of the challenging path planning instance with dynamic obstacles.
}
\label{fig:visual-abstract}
\end{figure}

Planning a path in the presence of both static and moving obstacles is a challenging problem with topical applications in robotics, video-games and other domains. 
When the environment is fully observable and the trajectories of the dynamic obstacles are known (e.g., are predicted by the robot's perception system or by a global observer), it is reasonable to account for them while planning. A prominent method that is tailored for such setting is Safe Interval Path Planning (SIPP)~\cite{phillips2011sipp}. This is a search-based algorithm that operates on a graph, which vertices correspond to the configurations of the agent (e.g., position, heading, velocity etc.) and edges correspond to the transitions between them. Due to the dynamic obstacles, certain vertices/edges of this graph are blocked at certain time intervals. SIPP accounts for that and constructs plans in which an agent can wait at the vertices to avoid collisions. An explicit assumption of SIPP is that the agent can start/stop moving instantaneously. Under this assumption the algorithm is provably complete and optimal (w.r.t. the given spatial and temporal discretization).

In our work we adopt a (more realistic from the practical perspective) assumption that decelerating to the full stop and, similarly, accelerating from the full stop takes time. In other words -- the agent is subject to kinodynamic constraints. This makes SIPP inapplicable as shown on Fig.~\ref{fig:visual-abstract}. In this example the environment is discretized to the $3 \times 10$ grid, where the distance between the centers of the grid cells is \SI{1}{\m}, and the agent has to go from the leftmost cell $B0$ to the rightmost one $B9$. Moreover two dynamic obstacles are present in the environment that block cells $B5$ and $B7$ during the intervals shown in black in the bottom part of the figure. Assume, the agent's cruising speed is \SI{1}{\m\per\s}. With this speed it takes \SI{1}{\s} to move between the cells. Moreover, the maximum acceleration/deceleration of the agent is \SI{0.5}{\m\per\s\squared}, which means that accelerating to the cruising speed from the full stop takes \SI{2}{\s} and \SI{1}{\m}. Therefore, when starting to move from the full stop the agent will arrive at the neighboring cell not in \SI{1}{\s} but rather in \SI{2}{\s}. If the algorithm ignores that, like the original SIPP, which assumes instantaneous acceleration, then the constructed plan, shown in red, will actually lead to the collision (if it is applied with real kinodynamic constraints) as shown by the dashed red line. 
A straightforward modification of SIPP that takes accelerating/decelerating actions into account is to apply only dynamically valid transitions, i.e., not allow to wait in the states where the velocity is not zero.
However, this variant will report \emph{failure} to find the solution while producing a partial plan shown in blue. In this plan the agent successfully stops before the first obstacle and waits for \SI{2}{\s} but then is unable to pass the second obstacle as it is being late at arriving at $B7$. The collision-free plan, however, exists. It is shown in green. The reason why SIPP fails to find it is that, intuitively, SIPP postpones the wait actions and cannot reason about the consequences of waiting in different states of the search tree. This does not violate the theoretical guarantees when waiting is available at any state. On the other hand, when waiting is not always readily available, i.e., the agent needs time to stop, the algorithm becomes evidently incomplete. We will elaborate more on this and provide technical details in the following sections of the paper.

To the best of our knowledge, no works on SIPP exist that directly address this issue. As a result, all known SIPP-based algorithms are incomplete in the setting when kinodynamic constraints of the agent have to be taken into account when planning. Our work fills this gap and presents an algorithm that is provably complete and optimal under such constraints -- Safe Interval Path Planning with Interval Projection (SIPP-IP). To empirically evaluate it, we conducted a wide range of experiments in which we compared SIPP-IP to several baselines that include other (non-complete) variants of SIPP one may think of and A*. Empirical results clearly show that SIPP-IP outperforms them all as it is able to solve instances that are unsolvable by other planners and its runtime is two order magnitude lower that the one of A*.

\section{Related Work}

Search-based planning with predictably moving obstacles can be straightforwardly implemented as A*~\cite{hart1968formal} with discretized time. Taking time dimension into account, however, leads to a significant growth of the search space, especially when the fine-grained time discretization is needed. To this end, in~\cite{phillips2011sipp} SIPP was introduced, based on the idea of compressing sequences of time steps into the time intervals and searching over these intervals. SIPP is provably complete and optimal under several assumptions including that the agent can start/stop moving instantaneously. Later numerous variants of SIPP emerged, enhancing the original algorithm, e.g., any-angle SIPP~\cite{yakovlev2021towards}, anytime SIPP~\cite{narayanan2012}, different variants of bounded-supotimal SIPP~\cite{yakovlev2020revisiting}. Moreover SIPP and its modifications are widely used as building blocks of some of the state-of-the-art multi-agent pathfinding solvers~\cite{cohen2019optimal,li2022mapf}. All these variants do not consider kinodynamic constraints.

Next we mention only the works that to some extent deal with taking agent's kinematic and/or kinodynamic constraints into account. In~\cite{ma2019lifelong} SIPPwRT was introduced that allowed planning with different velocities. Still, acceleration actions and effects were not considered in this work. Similarly, in~\cite{yakovlev2019prioritized} any-angle variant of SIPP was described that supported non-uniform velocities. Interestingly, the authors evaluate their planner on real robots. To plan safe trajectories for them they suggested inflating the sizes the moving obstacles, which alternatively can be seen as extending the blocked time intervals of certain graph vertices/edges. This ad-hoc technique was shown to perform reasonably well in practice, but in general it raises the question on how to choose the offset. For example, if one enlarges the blocked intervals in the setting described in the Introduction (recall Fig.~\ref{fig:visual-abstract}) by \SI{1}{\s} SIPP will not find a solution. In~\cite{10.1007/978-3-030-87725-5_1} accelerating actions were straightforwadly integrated into SIPP, which, again, makes the algorithm incomplete (as we show in this work). In~\cite{cohen2019optimal} a variant of SIPP was suggested for the problem setting which assumes arbitrary motion patterns, however this was not the main focus of the paper. Consequently, no techniques were proposed to take special care of the accelerating motions and the empirical evaluation considered only the motions that start and end with zero velocity. Overall, to the best of our knowledge no SIPP variant existed prior to this work that is provably complete and optimal when the assumption of the original SIPP that ``the robot can start/stop instaneously'' does not hold. 

\section{Problem Statement}

We assume a discretized timeline $T= 0, 1, ... $. 
The agent is associated with a graph $G=(V, E)$, with the start and goal vertices: $start, goal \in V$. A vertex of the graph corresponds to the state of the agent, alternatively known as the configuration, in which the agent can reside without colliding with the static obstacles. Each configuration explicitly encompasses the velocity, $vel$, of the agent. E.g., it may be comprised of the agent's coordinates and orientation as well as of its velocity: $v=(x, y, \theta, vel)$. Configurations with the same position/orientation but different velocities, indeed, correspond to different vertices of $V$. 
Configurations with $vel=0$ are of special interest as the agent may stay put (wait) only in them.

An edge $e=(v, u) \in E$ represents a \textit{motion primitive} \cite{pivtoraiko2011kinodynamic} -- a small kinodynamically feasible fragment of the agent's motion that transfers the agent from $v$ to (distinct) $u$. We assume that for any $v \in V$ a finite set of such motions is available and each motion takes integer number of time steps. Practically-wise this assumption is mild, as any real-valued duration may be approximated with a certain precision and represented as an integer. The duration of the motion defines the weight of the edge, $w(e) \in \mathbb{N}$.

Based on the source/target velocity, each motion (edge) can be classified as either accelerating, decelerating or uniform. In the latter case the agent's velocity at the target configuration is the same as in the source one, while in the first two cases it changes (increases, decreases respectively). 
The presence of accelerating/decelerating motions makes our problem distinguishable from the works, which assumed instantaneous acceleration.\footnote{Please note that our problem formulation, as well as the suggested method, is agnostic to how the state variables, e.g., velocity, change while the agent is executing a motion primitive. We are interested in the values of the state variables only at the endpoints of a motion primitive.}

Besides the agent and the static obstacles, a fixed number of moving obstacles populate the environment. We assume that their trajectories are known and converted (by an auxiliary procedure) to the collision and safe intervals associated with the graph vertices and edges, i.e., for each vertex, a finite sequence of non-overlapping time intervals $SI(v)$ is given, which is the sequence of the \emph{safe intervals} at which the agent can be configured at $v$ without colliding with any moving obstacle. Similarly, safe intervals are defined for each edge, $SI(e)$. If the agent executes a move $e$ at any time step outside $SI(e)$, a collision with a moving obstacle occurs.

A (timed) path for the agent, or a trajectory, is a sequence $\pi=(e_1, t_1), (e_2, t_2), ... (e_L, t_L)$, where $t_i \in T$ is the time moment the motion defined by the edge $e_i$ is started. The cost of the trajectory is the time step when the agent reaches the final vertex, $cost(\pi)=e_L + w(e_L)$. A trajectory is \emph{feasible} if the target vertex of $e_{i}$ matches the source vertex of $e_{i+1}$ and $t_{i+1} \geq t_{i} + w(e_i)$ (with $t_0=0$). Moreover if $t_{i+1} > t_{i} + w(e_i)$, then the velocity at the source of $e_{i+1}$ must be zero.

A feasible trajectory can also be seen as a \emph{plan} composed of the move and wait actions, where the move actions are defined by the graph edges and wait actions might occur at the vertices where the agent's velocity is zero. The duration of the wait action occurring after the $i$-th move action is computed as $\delta = t_{i+1} - (t_i + w(e_i))$. As $w(e_i)$ is integer, $\delta$ is integer as well.

A collision between the trajectory $\pi$ and the dynamic obstacles occurs \emph{iff} either of the conditions occur: \emph{i}) the agent starts executing some motion primitive $e$ at the time step which is outside of $SI(e)$; \emph{ii}) according to $\pi$ the agent waits at a vertex $v$ outside of $SI(v)$.

The problem now is to find a feasible collision-free trajectory $\pi$ that transfers the agent from $start$ to $goal$ on a given $G$ (with the annotated safe intervals of vertices/edges). In this work we are interested in solving this problem optimally, i.e., in reaching the goal as early as possible.

\section{Method}

As our method relies on SIPP, which in turn relies on A* with time steps, we first discuss the background and then delve into the details of the suggested approach. We assume that a reader is familiar with vanilla A* for static graphs. 

\subsection{Background}

\paragraph{A* with time steps} The straightforward way to solve the considered problem is to use A* for searching the state space whose nodes are the tuples $n=(v, t)$, where $v$ is the graph vertex and $t$ is the time step by which it is reached. When expanding a node, the successors are generated as follows. First, for each $e=(v,u)$, if $e$ is feasible, the node $n_{move}=(u, t+w(e))$ is added to the successors. Indeed, if the application of a motion $e$ at time $t$ results in a collision, the resulting node is discarded. Second, if $v$ allows waiting then the node corresponding to the wait action of the minimal possible duration of one time step $n_{wait}=(v, t+1)$ is also added to successors. Again, if the safe intervals of $v$ do not cover the time step $t+1$, this node is discarded.

Other parts of the search algorithm are exactly the same as in A*. It is noteworthy that the $g$-values of the nodes equal the time components of their identifiers. I.e., the $g$-value of the node $n=(v,t)$ equals $t$, as this is the minimal known estimate of the cost (w.r.t to the current iteration of the algorithm) form $start$ to $v$.

\paragraph{SIPP} In many scenarios A* with time steps generates numerous nodes of the form $(v, t), (v, t+1), (v, t+2)$ etc., which are created via the use of the atomic wait actions. This leads to the growth of the search tree and slows down the search. To this end, SIPP relies on the idea of compressing the sequential wait actions to reduce the number of the considered search nodes.

The search nodes of SIPP are identified as $n=(v, [lb_j, ub_j])$, where $[lb_j, ub_j]$ is the $j$-th safe interval of $v$.  
For each node SIPP stores the \emph{earliest arrival time}: $t \in [lb_j, ub_j]$. Whenever a lower-cost path to $n$ is found, $t(n)$ is updated and this value serves as the $g$-value of the node similarly to A* with time steps.

When expanding a node, SIPP does not generate any separate successor corresponding to the wait action. Instead, it generates only the successors that correspond to the ``wait and move'' actions that land into the neighboring configurations within their safe intervals. This means that to generate a successor, SIPP waits the minimal amount of time possible so that when the move to the new configuration is used, we arrive in the new safe interval as early as possible. For example, if SIPP performs a move from a node $n=(v, [5, 10])$ to a node $n'=(v', [15, 18])$ and the duration of that move is $5$ and $t(n) = 7$ then $t(n')$ is $15$, meaning that the agent after reaching $n$ at $t=7$ waits there for $3$ time steps and starts moving at $t=10$ and arrives to $n'$ at $t=15$. This move cannot be performed earlier than $t=10$ as in this case it will end outside the safe interval of $n'$.

SIPP is much faster than A* and is provably complete and optimal under the following assumptions: inertial effects are neglected, the agent can start/stop moving instantaneously and, consequently, the wait actions are available at all configurations.

\paragraph{What if the wait actions are not always available}

Intrinsic SIPP assumption that the agent can wait at any configuration often does not hold in practice as many mobile agents cannot start/stop instantaneously. To reflect this the accelerating/decelerating motions should be introduced, as well as the velocity variable should be added to the agent's configuration (as done in our problem statement). In this setting, SIPP is incomplete.

\begin{figure}
  \centering
  \includegraphics[width=0.75\linewidth]{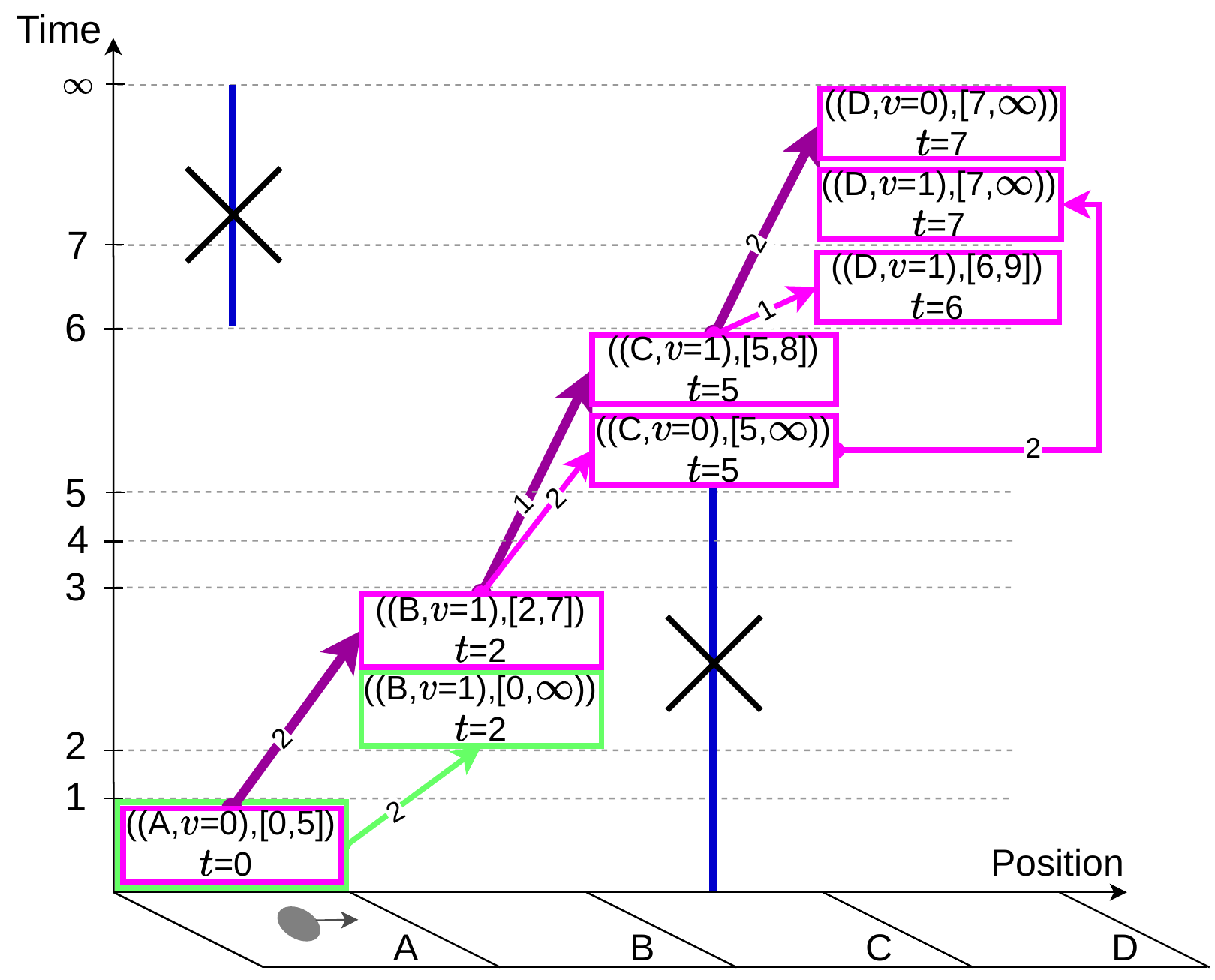}
\caption{
An example where the standard SIPP fails to find a solution, while SIPP-IP succeeds. The search tree of SIPP is shown in green. The search tree of SIPP-IP -- in purple.
}
\label{fig:sipp-fails-sippip-succeds}
\end{figure}

\begin{statement}

SIPP is incomplete when the agent is subject to kinodynamic constraints and wait actions are not available at any configuration.
\end{statement}

\begin{proof}
We prove this statement by presenting an example for which SIPP is not able to find a solution whilst it exists. Consider a problem depicted in Fig.~\ref{fig:sipp-fails-sippip-succeds}. The agent needs to reach cell $D$ from cell $A$. The motion primitives are defined as follows:
\begin{itemize}
    \item (accelerating motion) The agent starts moving from a cell with $vel=0$ and ends with $vel=1$ in the neighboring cell. The cost (duration) is $2$ time steps.
    \item (uniform motion) The agent starts moving from a cell with $vel=1$ and ends with $vel=1$ in the neighboring cell. The cost (duration) is $1$ time step.
    \item (decelerating motion) The agent starts moving from a cell with $vel=1$ and ends with $vel=0$ in the neighboring cell. The cost (duration) is $2$ time steps.
\end{itemize}

Indeed, the agent can wait when $vel=0$. The duration of the wait action is $1$ time step. For the sake of brevity, we ignore the agent's orientation and motions that change it.

SIPP starts with expanding the initial search node $((A, vel=0), [0,5])$ in which only one motion primitive, i.e., the acceleration motion, is applicable. This results in generating the successor node $((B, vel=1), [0, \infty))$ with the arrival time $t=2$, shown in green in Fig.~\ref{fig:sipp-fails-sippip-succeds}. The agent cannot wait in this configuration (as its velocity is not $0$). Thus, it can only continue moving arriving to $C$ either at $t=3$ using the uniform motion action or at $t=4$ using the decelerated motion. Both moves are invalid as they reach $C$ not within its safe interval which starts at $5$. Thus no successors are generated and no search nodes that can be expanded are left in the SIPP's search tree. The algorithm terminates failing to report a solution. The latter, however exists: the agent needs to wait for $2$ time steps at $A$ and then continue moving towards $D$. In this way it will arrive to $C$ at $t=5$ satisfying the safe interval constraint.
\end{proof}

Next, we present a modification of SIPP that is complete and optimal under the considered assumptions.

\subsection{SIPP-IP: Safe Interval Path Planning With (Wait) Interval Projection}

\paragraph{Idea} The main reason why standard SIPP fails to solve planning instances like the one presented above is because the information about possible wait actions is not propagated from predecessors to successors. In the considered example, when achieving the search node $((B, vel=1), [0, \infty))$ SIPP ``forgets'' that the agent can wait in the predecessor and perform the move action at any time step until the end of the predecessor's safe interval. This is not a problem when the agent can wait at any configuration, but leads to incompleteness in the case we are considering.

To this end, we substitute the safe interval as the search node's identifier with another time interval which we refer to as the \emph{waiting interval}. For any search node, this interval belongs to the safe interval of the graph vertex. Indeed, nodes with the same vertex but different waiting intervals should be distinguished. Waiting interval of a search node incorporates the information about all possible wait-and-move actions that can be performed in its predecessor. When a node is expanded, the waiting interval is \emph{projected} forward to all of its successors, thus the information about the possible wait-and-move actions is propagated from the root of the search tree (start node) along all of its branches. We call the modification of SIPP that implements this principle -- Safe Interval Path Planning With (Wait) Interval Projection, SIPP-IP. In the rest of the paper we will refer to the waiting interval of a SIPP-IP node as time interval (or, simply, interval). When talking about the safe intervals of the graph vertices we will never omit ``safe'' to avoid confusion.

\begin{figure}[t!]
  \centering
  \includegraphics[width=0.5\linewidth]{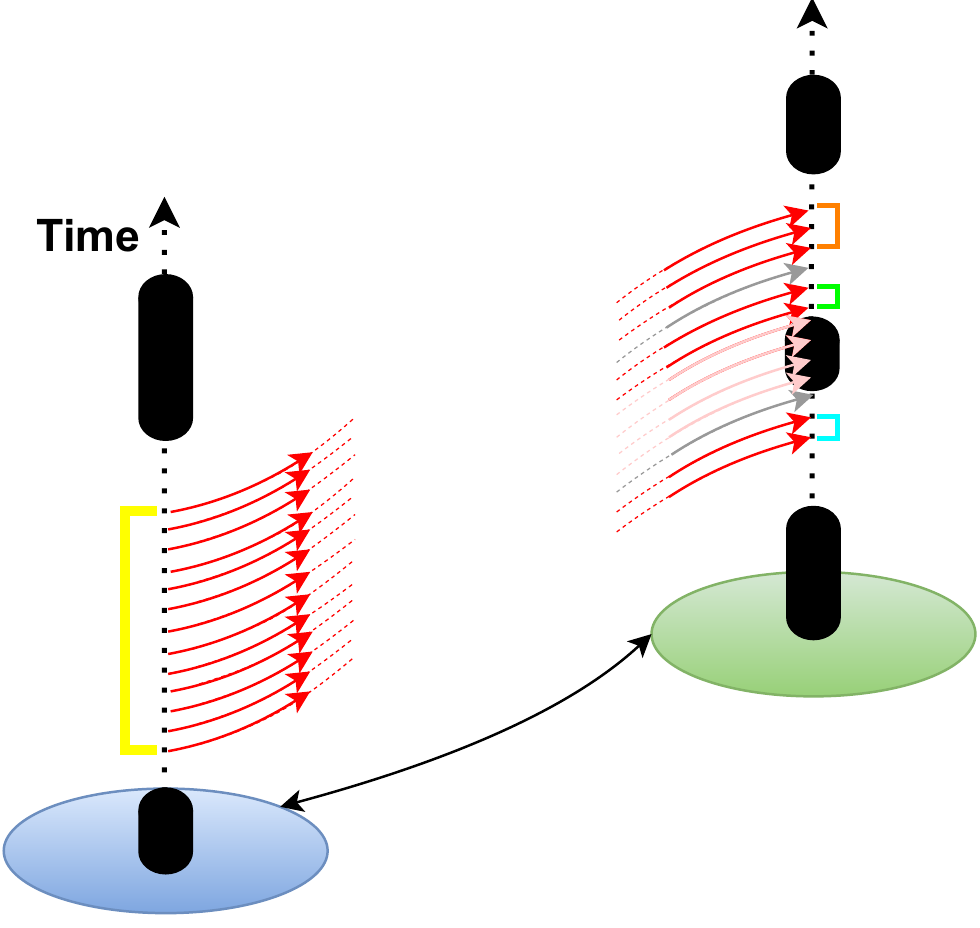}
\caption{Projection of the time interval in SIPP-IP.
}
\label{fig:time-interval-projection-general}
\end{figure}

\paragraph{Projecting intervals}

The role of the projection operation is to propagate the information on all of the available wait-and-move actions from the predecessor to successor. Formally, the input for the projection procedure is a SIPP-IP search node, $n=(v, [t_l, t_u])$, where $[t_l, t_u]$ resides inside one of the safe intervals of $v$, and a graph edge $e=(v, v')$ along which the interval should be propagated. The output is the set of time intervals $TI = \{ ti=[t', t'']\}$, s.t.:
\begin{itemize}
    \item each resultant interval belongs to one of the safe intervals of the target vertex: $\forall  ti_k\in TI~ \exists si \in SI(v'): ti_k \subseteq si$ 
    \item resultant intervals do not overlap: $ ti_k \cap ti_l = \emptyset, \forall k \neq l$
    \item for any transition, specified by $e$, that starts at any time step of the source interval, the time step corresponding to the end of this transition belongs to one of the resultant intervals if the transition is valid (no collisions occur when performing it): $\forall t \in [t_l, t_u]$ s.t. $(e, t)$ is a collision-free transition $\exists \hat{t} \in TI$ s.t. $\hat{t}=t+w(e)$
\end{itemize}

A general example of how projecting interval operation works is depicted in Fig.~\ref{fig:time-interval-projection-general}. The source graph vertex is denoted by the blue oval, the target vertex (where the outgoing edge lands) is shown in green. Black cylinders correspond to the blocked intervals of the vertices. The time interval to be projected is marked in yellow. The resultant time intervals are shown in cyan, green and orange. Observe, that due to the existence of the intermediate unsafe interval in the target vertex the projected time interval is, first, split into the two ones (pink transitions landing in unsafe intervals are ruled out). Second, the transitions that land in the safe intervals of the destination vertex but that lead to a collision in the course of the transition (gray arrows) are also pruned. Thus, the lower time interval is trimmed, while the upper is split into the two ones.

A straightforward technical implementation of the projecting operation may involve sequential application of $e$ to the time steps forming the input interval with further filtering out the invalid transitions and grouping the resultant time steps into the intervals. This resembles A* with time steps, however the resultant atomic search states of A* are compressed into the interval states of SIPP-IP, thus the search tree of SIPP-IP is more compact. \textcolor{blue}{Indeed, more computationally efficient implementations of the projection operation may be suggested, depending on how the collision detection mechanism is specified. 
One of such efficient implementations tailored to the grid-based representation of the problem presented originally in~\cite{cohen2019optimal} and used in our experiments (when each edge is associated with the sequence of cells the agent sweeps while execution the motion) is described in Appendix.} \footnote{\textcolor{blue}{The colored excerpt, as well as Appendix, is not likely to appear in the conference version of the manuscript.}}

To demonstrate how projecting helps in solving instances that were un-solvable for standard SIPP recall the example depicted in Fig.~\ref{fig:sipp-fails-sippip-succeds}. When expanding the start node SIPP-IP will project the time interval of $A$, which is $[0, 5]$ (coincides with the safe interval for start node), to the successor. This will result in $[2, 7]$ interval. When expanding the node $((B, vel=1), [2, 7])$  we can now apply both uniform motion action and decelerating action by trying to commit them at \emph{any} time moment that belongs to the time interval of $(B, vel=1)$ (via the projection operation, again). In such a way, the uniform motion action from $B$ to $C$ will be committed at $t=4$, following the SIPP's principle of reaching the successor as early as possible, i.e., at $t=5$. The projection of the waiting interval will give us $[5, 8]$. Finally, when expanding  $((C, vel=1), [5, 8])$ the node $((D, vel=0), [7, \infty))$ will be generated. When this node will be chosen for expansion the search will successfully terminate reporting finding the solution.

\begin{algorithm}[t!]
\caption{SIPP-IP}
\DontPrintSemicolon
\LinesNumbered
\label{alg:main}
\SetKwProg{Fn}{Function}{:}{}
\SetKw{Continue}{continue}
\SetKwFunction{FfindPath}{findPath}
\SetKwFunction{FgenSucc}{getSuccessors}

\Fn{\FfindPath{$v_{start}, t_{start}, v_{goal}, G(V,E), SI$}}{
    \nl OPEN $\leftarrow$ $\phi$, CLOSED $\leftarrow$ $\phi$\;
    \nl $ti=[t_{start}, t_{start}]$\;
    \nl \If{$v_{start}.vel=0$}{
        \nl $ti.t_u$ $\leftarrow$ upper bound of $SI(v_{start}, ti)$
    }
    \nl $n_{start} \leftarrow (v_{start}, ti)$\;
    \nl $f(n_{start}) \leftarrow t_{start} + h(v_{start})$\;
    \nl Add $n_{start}$ to OPEN\; 
    \nl \While{OPEN $\neq \phi$}{
        \nl $n \gets$ state from OPEN with minimal $f$-value\;
        \nl remove $n$ from OPEN, insert $n$ to CLOSED\; 
        \nl\If{$n.v = v_{goal}$}{
            \nl\Return $\pi$ $\leftarrow$ ReconstructPath($n, n_{start}$)\;
        }
        \nl succ $\leftarrow$ getSuccessors($n, G(V,E), SI$)\;
        \nl\For{each $n'$ in succ}{
            \nl \If{$n'$ in\ CLOSED\ or\ in\ OPEN}{
                \nl \Continue\;
            }
            \nl $f(n') \leftarrow n'.t_l + h(n'.{v})$\;
            \nl Add $n'$ to OPEN \;
        }
    }
    \nl \Return $\phi$
}
\Fn{\FgenSucc{$n, G(V,E), SI$}}{
    \nl SUCC = $\phi$\;
    \nl \For{each $e=(n.v,v')$ in available\ motions}{
        \nl     $intrvls$ = projectIntervals($n, e, SI$)\;
        \nl \If{$v'.vel=0$}{
            \nl \For{each $ti$ in\ $intrvls$}{
                \nl $ti.t_u\leftarrow$ upper bound of $SI(v',ti)$\;
            }
        }
        \nl \For{$ti$ in $intrvls$}{
            \nl insert $(v',ti)$ to SUCC\;
        }
    }
    \nl \Return SUCC\;
}
\end{algorithm}

\paragraph{SIPP-IP description} SIPP-IP starts with forming the start node from the given graph vertex. Initially, start interval contains only the first time step, $t_{start}$, provided as input. If the start vertex allows waiting, i.e., the velocity of the initial agent's configuration is zero, the upper bound of the node's interval is extended to the upper bound of the safe interval (of the start vertex), in which $t_{start}$ resides.

Then SIPP-IP follows the general outline of SIPP/A*. At each iteration it, first, selects the best node from $OPEN$, i.e., the one with the minimal $f$-value. The $f$-value of a SIPP-IP node $n=(v, [t_l, t_u])$ is defined as $f(n)=g(n)+h(n)=t_l+h(v)$. Similarly to SIPP, the $g$-value of the node is the earliest possible time step the agent can arrive an $n$, that is $t_l$. The $h$-value is independent of the time interval and is defined for graph vertices. It estimates the travel time from the current vertex to the goal one (e.g. it equals the straight-line distance between the vertices divided by the maximum speed of the agent). As in SIPP, we assume the heuristic to be admissible and consistent for SIPP-IP.

After selecting a node its successors are generated. For each outgoing edge we apply the projecting operation as described above. As a result for each edge we obtain a set of projected time intervals. If the target configuration allows waiting (the velocity is zero) than we extend the upper bounds of the projected intervals to the upper bounds of the corresponding safe intervals of the target vertex. As a result each of the generated successors implicitly encompasses, at the time interval, the information both about all possible transitions from the predecessor and all possible wait actions in the current node.

Finally each generated successor is inserted to $OPEN$ in case it is not already present in the search tree. 
The algorithm stops when a node corresponding to the goal vertex is extracted from $OPEN$. At this stage the path can be reconstructed. To do so, we go backwards from the goal node and at each iteration do the following. Let $n$ be the current node (initially set to the goal node, $n_{goal}$), $t$ -- the time variable (initially equal to $n_{goal}.t_{l}$), $n_{parent}$ -- parent of $n$, and $e$ the transition between them. If the agent can wait at $n$, then we add the tuple $(e,n.t_{l}-cost(e))$ to $\pi$ and change $n=n_{parent}$ and $t=n.t_{l}-cost(e)$. If the agent cannot wait at $n$ we add $(e,t-cost(e))$ to $\pi$ and change $n=n_{parent}$ and $t=t-cost(e)$. This is repeated until we reach the start node. If the final value $t$ does not match $t_{start}$, then the agent has to wait at start (for $t-t_{start}$ time steps).

The pseudo-code of SIPP-IP is presented in Algorithm 1.

\subsection{Theoretical Properties of SIPP-IP}

According to the definition of SIPP-IP state, we can view each SIPP-IP state as sequence of A*-with-Time-Steps (A*-TS) states. That is $n_{SIPP-IP}(v, [t_{l},t_{u}])$ = $\{s_{A*-TS}(v, t):t\in [t_{l},t_{u}]\}$. We will use this relation along the next proofs.

\begin{theorem}
SIPP-IP is complete.
\end{theorem}
\begin{proof}
In the presented function \textit{getSuccessors} we try to use all available edges on the input vertex. Then for each edge using the function \textit{projectIntervals} we get the intervals which contain all possible valid timesteps at which we can get at the target node of the edge starting from a timestep in the time-interval of the input state (by definition of \textit{projectIntervals}). This is equivalent to generate all A*-TS states from the A*-TS states included in the input state using an edge. By using all the edges, we generate all A*-TS which can be generated by the "move" action. Directly after that, we check if the wait action is available at the target vertex, then we extend the resulting time-intervals to the upper bound of the safe interval at the target vertex, where the interval is located. This is equivalent to the application of the wait actions on all A*-TS states at the target vertex.
As a result, in \textit{getSuccessors}, all valid A*-TS successors are generated (and capsulated by SIPP-IP states). So, SIPP-IP can be viewed as a modified version of A*-TS which at every iteration it expands several states at the same time and generates all the successors of these states and insert them into OPEN (by definition, all generated A*-TS states are reformed into SIPP-IP states without any loss). As a result SIPP-IP will always generate and expand all A*-TS states, and as A*-TS is complete, SIPP-IP is also complete.
\end{proof}

\begin{lemma}
\label{lemma:astar-state-at-bottom}
The A*-TS state with the minimum $f$-value in a SIPP-IP state $n_{SIPP-IP}(v, [t_{l}, t_{u}])$ is the state $n_{A*-TS}(v,t_{l})$.
\end{lemma}
\begin{proof}
Recall that $f$-value of a A*-TS state is equal to $n.t+h(n.v)$. As $v$ of all A*-TS states in one SIPP-IP state are identical, therefore the $h$-values of them are equal. As a result, the state with the minimal time i.e., $t_{l}$ is the state with the minimal $f$-value.
\end{proof}
\begin{theorem}
SIPP-IP is optimal.
\end{theorem}
\begin{proof}
As SIPP-IP is complete algorithm and the cost (time) is included as identifier in the state, it is guaranteed that the optimal state will be expanded. Therefore, it is sufficient to prove that the first expanded state with the goal vertex is the optimal one i.e., contains the A*-TS state with optimal (minimal) time. Let again view the SIPP-IP states as a sequence of extracted A*-TS states. Let the first expanded SIPP-IP state with goal vertex to be $n_{SIPP-IP}(v_{goal},[t_{l},t_{u}])$ which contains the A*-TS state $n_{A*-TS}(v_{goal},t_{l})$. Let us suppose that $n_{A*-TS}$ is not the optimal A*-TS state but there exists another state $n'_{A*-TS}$ with the minimal $f$-value in OPEN $f(n')$ which is less than $f(n)$. According to Lemma~\ref{lemma:astar-state-at-bottom}, $n'_{A*-TS}$ is located at the bottom of a SIPP-IP state $n'_{SIPP-IP}$ in OPEN i.e., $n'_{A*-TS}.t=n'_{SIPP-IP}.t_{l}$. As in SIPP-IP the states are ordered by the values $f(n_{SIPP-IP})=n.t_{l}+h(n.v)$ which is equal to the $f$-value of the bottom A*-TS state $f(n_{A*-TS})$, the state $n'_{SIPP-IP}$ should have been expanded before the state $n_{SIPP-IP}$ because $f(n'_{A*-TS})<f(n_{A*-TS})$ which result in contradiction. Therefore, there is no A*-TS state with $f$-value less than $f(n_{A*-TS})$ and hence $n_{SIPP-IP}$ is the optimal state.
\end{proof}

\begin{figure*}[t]
    \includegraphics[width=1\linewidth]{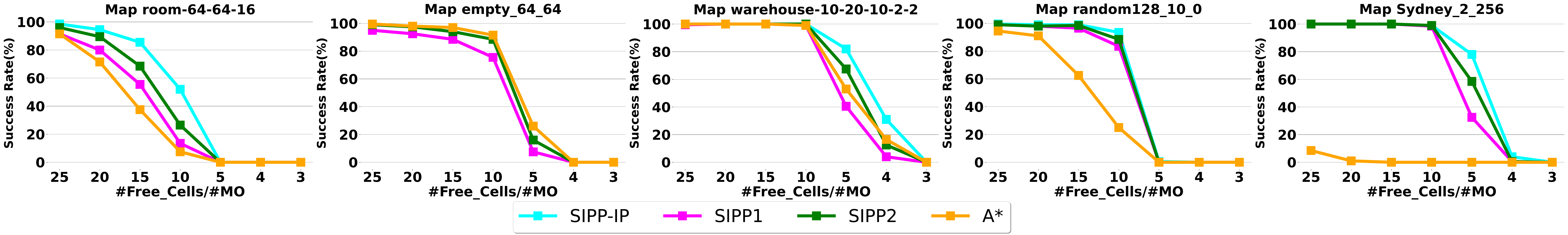}
    \caption{Success Rates of the evaluated algorithms.}
    \label{fig:tests1-solutions-srs}
\end{figure*}

\begin{figure}
\centering
    \includegraphics[width=1\linewidth]{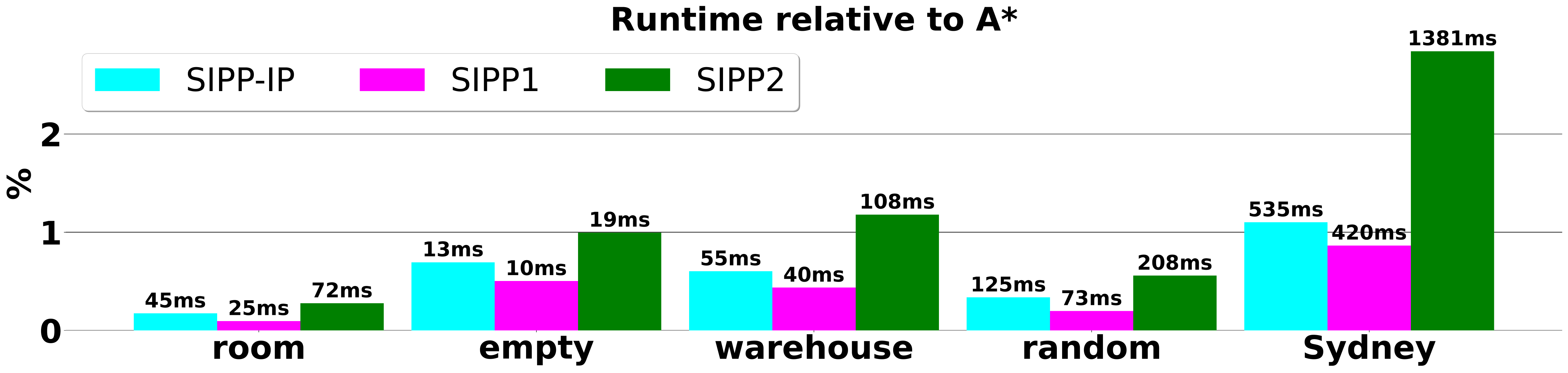}
    \caption{Runtimes of SIPP-IP, SIPP1 and SIPP2 compared to A*. Absolute values (in \SI{}{ms}) are shown above the bars.}
    \label{fig:tests1-solutions-times}
\end{figure}

\section{Empirical Evaluation}

We used five different maps of varying sizes and topologies from the MovingAI benchmark~\cite{sturtevant2012benchmarks} for the experiments: \texttt{empty} (sized 64x64), \texttt{room} (64x64), \texttt{warehouse} (84x170), \texttt{random} (128x128) and \texttt{Sydney} (256x256). Each map was populated with an increasing number of moving obstacles (MOs), and for each number of MOs $200$ different instances were generated that differ in trajectories of MOs. 
These trajectories were generated by randomly assigning start and goal locations for each MO and letting it go from start to goal 
using random speeds. Moreover, MO may wait at any cell for random number of time steps.
For each map we generated instances with the following densities of MOs:  1/25, 1/20, 1/15, 1/10, 1/5, 1/4, 1/3, where the density is the ratio of the number of MOs to the number of the free cells. The start and goal locations of the agent were fixed to the top-left and bottom-right corners of the map, respectively. Please note, that numerous instances in our dataset (especially with high densities of MO) are not generally solvable.

The agent was modeled as disk whose diameter equals the length of the grid cell. The configuration is defined as $(x,y,\theta,vel)$ where $x,y$ are the coordinates of the cell, $\theta\in\{0^\circ,90^\circ,180^\circ,270^\circ\}$ is the orientation and $vel\in\{0,2\}$ is the velocity of the agent.

The following kinodynamic motion primitives were defined: accelerating, decelerating, uniform. Accelerating (decelerating) primitive: move with the fixed acceleration of \SI{0.5}{cell\per\s\squared} (\SI{-0.5}{cell\per\s\squared}) from a configuration with zero (maximum of \SI{2}{cell\per\s}) velocity until reaching maximum (zero) velocity in the direction of the initial orientation. The agent traverses 4 cells this way without changing its orientation. Uniform motion primitive: go with the maximum velocity one cell forward i.e., the initial and final velocities equal \SI{2}{cell\per\s} and the orientation will be the same but the cell will change by one according to the orientation. Additionally, we considered in-place rotations and waiting actions. If the velocity is zero, the agent can rotate \SI{90}{\degree} in \SI{2}{\s} or wait for one time step. The time step is chosen to be \SI{0.1}{\s}.

Collision checking is conservative. We prohibit the agent and any MO from touching the same cell at the same time. Therefore, the blocked intervals of any cell were reserved for the MO when it touches this cell even partially. The agent is allowed to use a motion primitive only if it will not result in touching a cell in a blocked interval. Moreover, if the time-limits of touching some cell for an agent are not integers, they are extended to the closest integers in the safer manner.

We compared SIPP-IP to A* with time steps and two straightforward (not complete) extensions of SIPP: SIPP1 and SIPP2. SIPP1 is a modification of SIPP that generates only kinodynamically-feasible successors for the configurations where velocity is not zero. SIPP2 further allows to re-expand the search nodes (which corresponds to allowing to reach configurations with non-zero velocities in different time steps in the same safe interval). When implementing SIPP1/SIPP2 we used the techniques from SIPPwRT~\cite{ma2019lifelong}, thus the latter can be considered to be included in comparison. The C++ source code of all planners is publicly available. \footnote{\url{https://github.com/PathPlanning/SIPP-IP}}

The experiments were conducted on a PC with Intel Core i7-10700F CPU @ 2.90GHz $\times$ 16 and 32Gb of RAM. We imposed a limit of $100,000,000$ of generated nodes for all algorithms. For each instance we tracked whether the algorithm produced a solution and recorded the algorithm's runtime and the solution cost.

Fig.~\ref{fig:tests1-solutions-srs} presents the Success Rate (SR) plots, where SR is the ratio of the successfully solved instances to all of our test instances. Indeed, all SIPP-IP competitors have lower SR compared to SIPP-IP (except the empty map, where A* and SIPP-IP have the same SR). For SIPP1 and SIPP2 this is explained by their incompleteness. In theory SR of A* should always match the one of SIPP-IP, however in practice the search space of A* is large and it exceeded the imposed limit of $100,000,000$ generated nodes in numerous runs.

For each map we also analyzed the cost of the instances that were successfully solved by SIPP-IP, SIPP1 and SIPP2. The results are presented in Table~\ref{table:costs}. Each cell in the table tells in how many instances (in percent) the cost of the SIPP1/SIPP2 solution exceeded the cost of the optimal solution found by SIPP-IP by a fixed factor. Evidently, in most of instances the costs of competitors are larger than of SIPP-IP's costs and for certain maps (e.g., room) the percentage of the instances where their cost notably exceeds (by more than 50\%) SIPP-IP's cost is significant (up to 42\%).

\begin{table}[t]
\resizebox{\columnwidth}{!}{
\begin{tabular}{ c|c|c|c|c|c|c } 
 &room & empty & warehouse & random &  Sydney & Factor\\ 
 \hline
 SIPP1&93\% &82\% & 69\%&89\% &97\% & \multirow{2}{*}{$>$}\\
 SIPP2&90\%&79\% & 66\%& 85\%&92\%&\\
  \hline
 SIPP1&77\% & 46\% & 16\% & 55\% & 30\% & \multirow{2}{*}{$>5\%$}\\ 
 SIPP2&60\% & 39\% & 11\% & 44\% & 20\%&\\ 
 \hline
 SIPP1&42\% & 7\% & 1\% & 7\% & 4\% & \multirow{2}{*}{$>50\%$}\\ 
 SIPP2&22\% & 3\% & 0\% & 1\% & 3\%&\\ 
\end{tabular}
}
 \caption{Percentage of SIPP1/SIPP2 solutions that have higher costs compared to SIPP-IP.}
 \label{table:costs}
\end{table}

Finally, the algorithms' runtime analysis is presented in Fig.~\ref{fig:tests1-solutions-times}. Indeed all versions of SIPP, including the complete one -- SIPP-IP, are  significantly faster than A*, and reduce the computation time by two orders of magnitude. Moreover SIPP-IP is faster than SIPP2. Still, it is outperformed by SIPP1. This is expected, as SIPP1 exploits a straightforward expansion strategy missing numerous successors SIPP-IP would generate. Absolute (averaged) values of the runtime of SIPP-IP (less than \SI{0.1}{s} on all maps, except the largest one) suggest that the proposed planner can be utilized in real robotic systems, where taking into account kinodynamic constraints may be of vital importance.

\section{Conclusion}
In this work we presented a provably complete and optimal variant of the prominent Safe Interval Path Planning algorithm capable to handle kinodynamic constraints. We showed that straightforward ways to extend SIPP fail in the considered setup, thus a more involved algorithm is needed. The later was presented and analyzed theoretically and empirically.
The directions for future research include embedding the suggested algorithm within the multi-agent path planning solver and conducting experiments on real robots.

\section{Acknowledgments}
This work was partially supported by a grant for research centers in the field of artificial intelligence, provided by the Analytical Center for the Government of the Russian Federation in accordance with the subsidy agreement (agreement identifier 00000D730321P5Q0002) and the agreement with the Moscow Institute of Physics and Technology dated November 1, 2021 No. 70-2021-00138.

\bibliography{main.bib}
\clearpage
\newpage
\appendixpage
\appendix

\paragraph{Grid representation of the environment.} Grids are commonly used to represent the environment in a wide range of applications (robotics, video games, etc.). In our evaluation of SIPP-IP we also used grids. More specifically the environment is represented by a set of grid cells $\mathcal{C}$, where each cell can be either traversable or blocked due to the static obstacles. Each vertex $v$ of the input graph $G=(V, E)$ (i.e. the graph that specifies the agent's possible configurations and the transition between them) is now additionally associated with a set of cells that the agent touches when it is configured at $v$. 
Each edge $e(v,v') \in E$ connecting vertex $v$ with vertex $v'$ is associated with a list (sequence) of cells $(c^e_0, c^e_1, ...)$ that are swept by the agent while execution the motion $e$. Moreover, for each cell $c^e_i$ the time interval $[lb^e_i, ub^e_i]$ is given, during which this cell is swept by the agent while executing the motion primitive (relative to the starting time of the motion). The cost of the edge $cost(e)$ is also given and defined by the duration of the motion primitive.

The safe/un-safe intervals that arise from the moving obstacles are now tied to the grid cells. That is, for each cell $c$ a finite sequence of (non-overlapping) time intervals $SI^c$ is given, which is the sequence of the \emph{safe intervals} at which no obstacle sweeps through the cell. For unification purposes we assign $SI^c$ to $\emptyset$ for the cells occupied by the static obstacles. Inversion of the safe intervals defines the collision intervals for a cell. Thus, the safe interval of the graph vertex $v\in V$ can be defined as the intersection of the safe intervals of the cells that the agent occupies while in $v$. 

Next we present the details of how the \emph{ProjectIntervals} procedure could be efficiently implemented for grids.

\paragraph{Projecting the time intervals in grids.}
Observe the example depicted in Fig.~\ref{fig:time-interval-projection}. Assume we have a SIPP-IP state $n=(v, [2,17])$. Moreover assume that, when configured at $v$, the agent touches only a single cell.\footnote{The projection procedure we are about to describe is agnostic, however, to how many cells the agents touches at an arbitrary configuration.} Consider a motion primitive $e(v, v_{end})=(c^e_0,c^e_1,c^e_2)$:~$[lb^e_0=0,ub^e_0=3],[lb^e_1=2,ub^e_1=4],[lb^e_2=3,ub^e_2=5]$ which means that going by primitive $e$, the agent sweeps through the cell $c^e_0$ from time step $0$ till time step $3$. It will also first touch cell $c^e_1$ at time step $2$ and continue sweeping through this cell for $2$ time steps and first touch cell $c^e_2$ at time step $3$ and continue sweeping $c^e_2$ for $2$ time steps as well. Please note, that the agent might be sweeping trough the several grid cells at the same time step, which fits the implicit assumption that the agent has body/shape. The cost of this motion primitive is $5$, meaning that the agent will arrive at the final configuration, $v_{end}$ after 5 time steps from the start of the move.

\begin{figure}[t!]
  \begin{subfigure}[c]{0.25\columnwidth}
  \centering
    \includegraphics[width=0.95\textwidth]{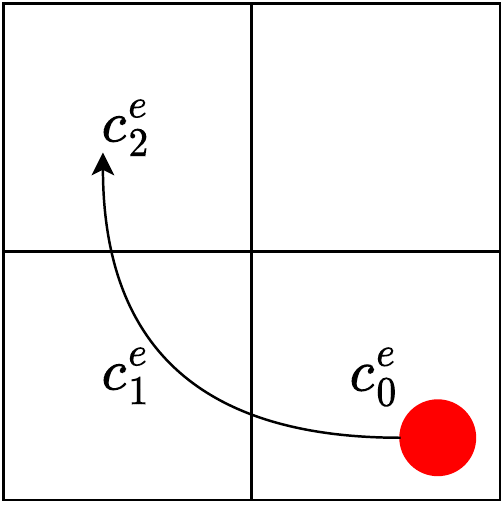}
  \end{subfigure}
  ~
  \begin{subfigure}[c]{0.65\columnwidth}
    \includegraphics[width=0.95\textwidth]{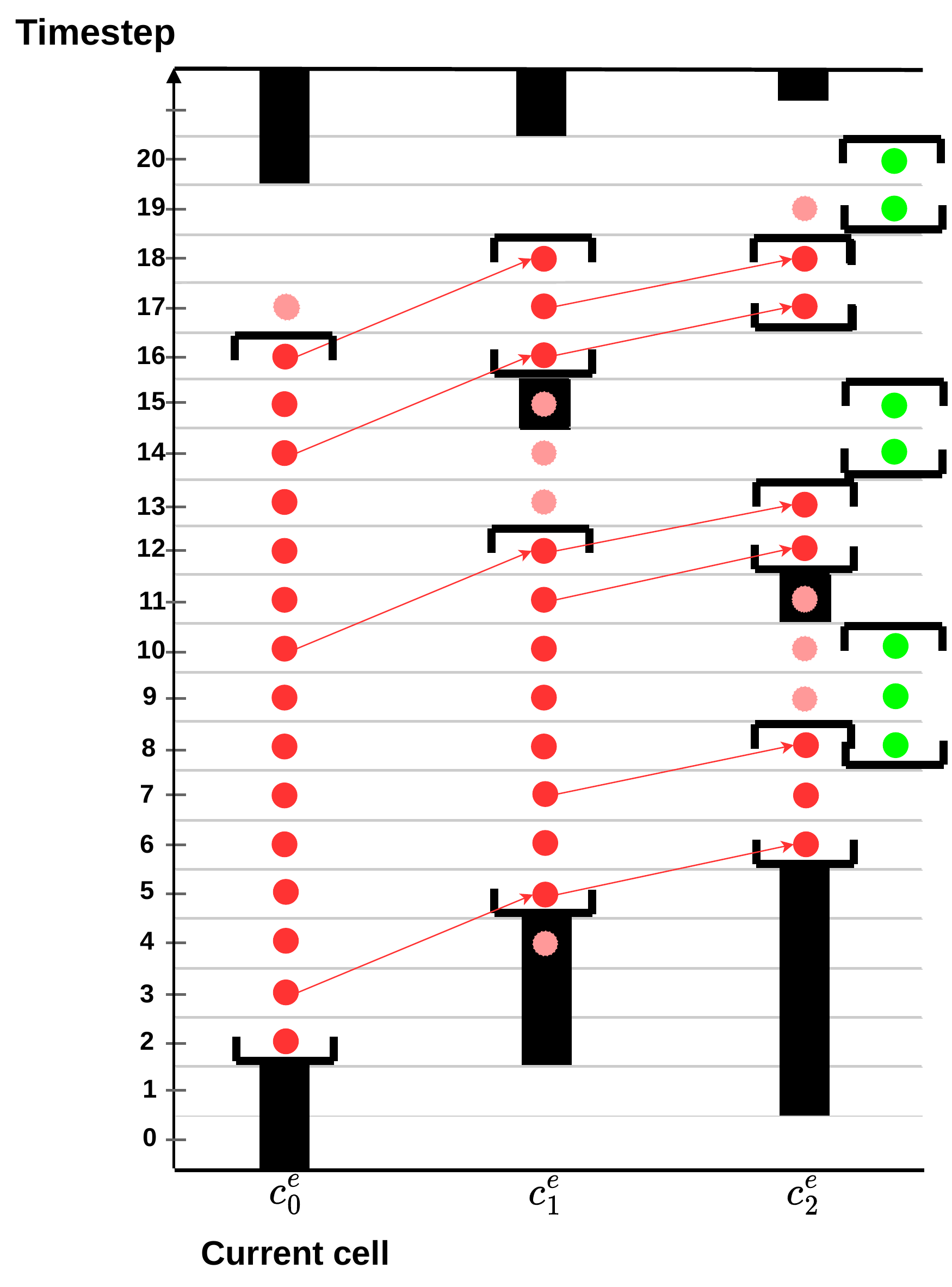}
  \end{subfigure}
\caption{
\emph{Left}: A motion primitive that sweeps through $3$ grid cells. \emph{Right}: Sequential projection of the time interval trough the cells forming the motion primitive.}
\label{fig:time-interval-projection}
\end{figure}

The task is to find the time intervals that contain all the time steps at which the agent can get to $v_{end}$ via executing $e$. Obviously, the naive way is to simulate the move of the agent from each time step in the source interval $[2,17]$ and compute which time steps will survive when going through $(c^e_0, c^e_1, c^e_2)$. Then one can reconstruct the time intervals from these time steps. 

However, we can compute the resultant time intervals in a more efficient way. Taking advantage from the sequential representation of the primitive, we can project the the initial time interval sequentially through the cells forming the primitive until the final cell is reached. Within this process only the manipulations with the endpoints of the time intervals need to be done (not all the time steps that form the time interval). At each step of this sequential cell-to-cell projection we are interested in getting the intervals that 1) contain all time steps at which the agent will first touching the destination cell 2) contain only such time steps that are valid for continuation of the movement via the given primitive. By the latter we mean that moving by the primitive from any time step inside these intervals will not lead to a collision with any un-safe interval of the cell under consideration. To verify that the second condition holds for the start cell of the primitive we technically project the initial time interval on this cell as well (before projecting to the next cells).

Overall, after the sequence of the cell-to-cell projections is performed we end with computing the time intervals of the final cell of the move ($c^e_2$ in our example) at which the agent first touches it. Finally, to get the time steps at which the agent finishes the move, we should shift these time steps to the end of the primitive.

In the considered example (recall Fig.~\ref{fig:time-interval-projection}) we start with projecting the interval $[2,17]$ on cell $c^e_0$. It results in the interval $[2,16]$. Please note that the latter does not coincide with the initial interval. The reason is that the time step $17$ is not a valid time step to start the primitive from. If the motion is started at $t=17$ then at $t=20$ the agent will still be touching $c^e_0$ (as the sweeping interval of $c^e_0$ lasts for $3$ time steps according to $e$). However, this is not valid w.r.t. to given safe intervals of $c^e_0$.

Next, projecting on the next cell $c^e_1$ results in two intervals $[5,12]$ in the first safe interval and $[16,18]$ in the second safe interval.
This means that the agent can first touch $c^e_1$ in the time steps $t\in [5,12] \cup [16,18]$. Time steps $13$ and $14$ are not allowed because the agent will also sweep for 2 time steps through cell $c^e_1$ after it first touches it according to the primitive $e$, so arriving at time step $13$ or $14$ will make the agent collide with the blocked interval, as time step $15$ is blocked at $c^e_1$. 
Next, projecting the interval $[5,12]$ results in interval $[6,8]$ in the first safe interval in $c^e_2$ and interval $[12, 13]$ in the second safe interval. Similarly, projecting the interval $[16, 18]$ on the next cell results in interval $[17,18]$ in the second safe interval of $c^e_2$.
Now we have all the time steps at the moment of touching cell $c^e_2$  which are safely (along all the cells) accessible from the input state. The function is assumed to return the time steps at the end of primitive (not at the moment of touching the last cell), so what we need to do finally is to only shift the resulting time steps by the duration $(cost(e)-lb_{c^e_2})=2$, so the resulting intervals become $[8,10], [14,15]$ and $[19,20]$.
\begin{algorithm}[t!]
\caption{Projecting Intervals}
\DontPrintSemicolon
\LinesNumbered
\label{alg:projectIntervals}
\SetKwProg{Fn}{Function}{:}{}
\SetKwFunction{projectIntervals}{projectIntervals}
\Fn{\projectIntervals{$n, e=(n.v,v'), SI$}}{
\nl$time\_ints \leftarrow \{[n.t_{l}, n.t_{u}]\}$\;
\nl $t \leftarrow 0$\;
\nl\For{each $cell$ in $e.cells$ consecutively}{
\nl    $new\_ints \leftarrow \phi$\;
\nl $\Delta \leftarrow lb^e_{cell}-t$\;
\nl $t \leftarrow lb^e_{cell}$\;
\nl    \For{each $ti$ in $time\_ints$}{
\nl    \For{each $si$ in $SafeIntervals(cell)$}{
\nl $t_{earliest} \leftarrow max(ti.t_{l}+\Delta, lb(si))$\label{alg:line:te}\;
\nl $t_{latest} \leftarrow min(ti.t_{u}+\Delta$, $ub(si)-(ub^e_{cell}-lb^e_{cell}))$\label{alg:line:tl}\;
\nl        \If{$t_{earliest} \leq t_{latest}$\label{alg:line:condition}}{
\nl            insert $[t_{earliest}, t_{latest}] \rightarrow new\_ints$\;
            }
        }
    }
\nl    $time\_ints \leftarrow new\_ints$\;
}
\nl$succ \leftarrow \phi$\;
\nl\For{each $ti$ in $time\_ints$}{
\nl $last\_cell \leftarrow$ the last cell in $e$\;
\nl $\Delta \leftarrow cost(e)-lb^e_{last\_cell}$\;
\nl $ti.t_{l} \leftarrow ti.t_{l}+\Delta$\label{alg:line:extend_to_lower}\;
\nl $ti.t_{u} \leftarrow ti.t_{u}+\Delta$\label{alg:line:extend_to_upper}\;
\nl    insert $(v',ti) \rightarrow succ$\;
}
\nl\Return $succ$\;
}
\end{algorithm}

\paragraph{Pseudocode} The algorithm starts by initializing the resulted time-intervals $time\_ints$ by the input time interval at Line 1. Then we iterate over all the cells of the given primitive sequentially (starting from the initial cell). At each iteration we first initialize a new empty set of time intervals $new\_ints$ to fill by the resulting projected intervals at the current cell $cell$. Then we calculate the time $\Delta$ needed to reach $cell$ from the previous cell or from the cell of the input vertex if $cell$ is the first cell in the primitive (using helper variable $t$). Then we start projecting the currently achieved intervals stored in $time\_ints$ on $cell$. This procedure can be done by finding the earliest time step $t_{earliest}$ and latest time step $t_{latest}$ at each safe interval in $cell$ from each interval in $time\_ints$. 
These time steps can be computed easily knowing the endpoints of the input time interval $ti$ and safe interval $si$ at which we project, as implemented in Lines~\ref{alg:line:te},~\ref{alg:line:tl}. Then if the resulting time steps form a valid interval (condition at Line~\ref{alg:line:condition}), we add it to $new\_ints$. At the end of each iteration we assign the new projected intervals $new\_ints$ to the $time\_ints$ to project them again on the next cells.
After the end of all iterations, we get the time interval of all time steps which we can get at the moment of first touching the last cell in $e$.
The last thing is to shift the resulted time intervals to the end of the primitive (i.e. shift them by the value $cost(e)-lb^e_{last\_cell}$) as can be shown in Lines~\ref{alg:line:extend_to_lower},~\ref{alg:line:extend_to_upper}. After that, we form successor states from the end vertex of $e$, $v'$ with each of the resulting time intervals and return them.

\end{document}